\documentclass{article}
\usepackage{makecell}
\usepackage{booktabs}
\usepackage{placeins}
\usepackage{amsmath}
\usepackage{wrapfig}

\usepackage[table]{xcolor}
\usepackage{array}

\DeclareMathOperator*{\argmin}{arg\,min}
\DeclareMathOperator{\msign}{msign}
\usepackage[preprint]{paper}

\usepackage[utf8]{inputenc} 
\usepackage[T1]{fontenc}    
\usepackage{hyperref}       
\hypersetup{%
    pdfborder = {0 0 0},
    colorlinks,
    citecolor=blue,
    linkcolor=blue,
}
\usepackage{url}            
\usepackage{booktabs}       
\usepackage{array}          
\usepackage{colortbl}       
\usepackage{amsfonts}       
\usepackage{nicefrac}       
\usepackage{microtype}      
\usepackage{xcolor}         
\usepackage{amsmath}
\usepackage{graphicx}
\usepackage[linesnumbered,ruled,vlined]{algorithm2e}
\usepackage{amsthm}
\usepackage{thmtools,thm-restate}

\newtheorem{lemma}{Lemma}

\usepackage{mathtools}

\title{Approximate Muon with low-rank adapters}

\author{%
Ben Anson\\
University of Bristol\\
School of Mathematics\\
\texttt{ben.lja@proton.me}
\And
Conor Houghton\\
University of Bristol\\
School of Engineering Mathematics and Technology\\
\texttt{conor.houghton@bristol.ac.uk}
\And
Edward Milsom\\
University of Bath\\
Department of Computer Science\\
\texttt{em846@bath.ac.uk}
}

\begin{document}

\maketitle

\begin{abstract}
The Muon optimizer shows clear benefits versus alternatives when pretraining neural networks. However, it is used less frequently for parameter-efficient fine-tuning (PEFT).
One potential reason is that the most common PEFT method, LoRA, does not naturally combine with Muon since it is not mathematically possible to orthogonalize the weight update given by a low-rank parameterization. In this paper, we address this issue by approximating the solution to a relaxed Muon objective in the low-rank setting via linearization and then least-squares. We provide an efficient implementation that uses matmul operations only, as opposed to more complex linear algebra decomposition routines. Our method, sMuon (small Muon), performs favourably across SFT and a ReLoRA pretraining experiment. While results are model- and eval-dependent, we find overall that using Muon for low-rank fine-tuning provides moderate performance improvements.\end{abstract}

\section{Introduction}
Muon~\citep{jordan2024muon} is a neural network optimizer which orthogonalizes updates at each layer, equalizing all singular values. This can be interpreted as steepest descent under the spectral norm \citep{bernstein2025modular}, thereby controlling the worst-case change to the layer as a linear operator, rather than viewing it as a flattened vector. Muon has demonstrated excellent results in the pretraining of LLMs, often outperforming AdamW, leading to its increasing adoption~\citep{liu2025muon,shah2025practical,team2025kimik2,team2026kimi,xu2026deepseek}.

However, Muon is less commonly used for fine-tuning (FT), particularly PEFT (Parameter-Efficient Fine-Tuning, which reduces compute resource requirements). There are various reasons for this, including the fact that Muon FT has been shown to perform better with Muon-pretrained base models~\citep{liu2025muon,qu2026can}, which are currently less common than AdamW-pretrained models. Perhaps the most glaring issue is that Muon does not straightforwardly combine with LoRA~\citep{hu2022lora}, the dominant PEFT method. In principle, Muon requires that we orthogonalize the descent step for the whole layer, but this is not generally realizable in the LoRA parameterization. Some approximation is therefore inevitable. The simplest and most common workaround is `per-factor Muon', which ignores the geometry of the whole layer and orthogonalizes changes to $B$ and $A$ separately~\citep{qu2026can,janson2026stabilizing,kang2026uniform}.
Other works retain the whole layer's geometry and introduce an approximation elsewhere,
with~\citet{cesista2026lora} relaxing the Muon objective using the triangle inequality, and \citet{bogachev2025lora} recasting the problem as Riemannian optimization on a low-rank manifold.
\begin{wrapfigure}{r}{0.5\linewidth}
  \centering
    \includegraphics[width=\linewidth]{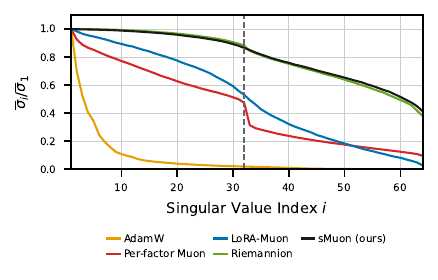}
    \caption{Spectra of the weight change for different LoRA optimizers. We show the average singular values of the weight updates (attention weight, layer 6 of 12), during ReLoRA pretraining ($r=32$). We normalize by $\overline{\sigma}_1$, and denote the LoRA rank by the dashed grey line.}
    \label{fig:mlp_svs}
\end{wrapfigure}
However, we feel that alternative approaches that more closely fit with Muon's original `steepest descent under the spectral norm' motivation are still underexplored.

In this paper, we linearize the low-rank Muon objective to find the optimal descent step within the available parameter space. This solution effectively orthogonalizes the full layer's gradient and projects it onto the appropriate low-rank subspace. Empirically, we find this retains a more uniform spectrum than alternatives such as per-factor Muon (Fig.~\ref{fig:mlp_svs}) and LoRA-Muon~\citep{cesista2026lora}. Furthermore, unlike Riemannion~\citep{bogachev2025lora}, our approach avoids expensive matrix decompositions. While our solution to the low-rank Muon problem is clean (Alg.~\ref{alg:conceptual}), a naive implementation is computationally inefficient, as it requires orthogonalizing a large matrix.
To address this, we introduce a mathematically equivalent formulation (Alg.~\ref{alg:matmul_only}) that is both fast and memory-efficient. By relying entirely on matrix multiplications, we are able to avoid the overhead from QR and SVD.
Finally, we show that our optimizer, sMuon, achieves competitive performance across a range of commonsense and coding SFT benchmarks, as well as in ReLoRA pretraining experiments.

Our contributions are as follows:
\begin{itemize}
  \setlength{\itemsep}{1pt}\setlength{\parskip}{0pt}
  \item We derive a method for training LoRA adapters with Muon by taking a first-order approximation of the spectral-norm-regularized steepest descent objective. The method is available in closed form, and requires less optimizer state than AdamW.
  \item We provide a matmul-only algorithm of the method, sMuon, that is efficient for modern hardware, with comparable speed to other Muon methods such as per-factor Muon and LoRA-Muon, while being significantly faster than Riemannion for larger LoRA rank.
  \item We run SFT experiments on four base models, three AdamW-pretrained and one Muon-pretrained, over eleven benchmarks to compare low-rank Muon optimizers. While results are model-dependent, we find that our optimizer gives strong performance, performing best on the Muon-pretrained model with the top accuracy in six out of eleven tasks with Moonlight-16B-A3B.
  \item In the ReLoRA pretraining setting, we illustrate that Muon LoRA methods are generally stronger than AdamW LoRA. Specifically we find that our method attains the lowest final validation loss, tying with Riemannion, but being $\sim\!30\%$ faster.
\end{itemize}

\FloatBarrier
\section{Related Work}
\paragraph{Muon and spectral steepest descent}
Muon orthogonalizes the gradient descent step for each weight matrix, so that the singular values of weight updates are equal~\citep{jordan2024muon}. This update is exactly steepest descent under the spectral norm~\citep{bernstein2024old,bernstein2025deriving}. Muon has been shown to scale to billion- and trillion-parameter models for LLM pretraining, with considerable efficiency gains over AdamW~\citep{liu2025muon,team2025kimik2}. 

\paragraph{Muon for fine-tuning and optimizer mismatch}
Although Muon excels at pretraining, it is used far less for fine-tuning. \citet{liu2025muon} note that fine-tuning an AdamW-pretrained model with Muon (or vice-versa) degrades performance. \citet{qu2026can} suggest that the differing implicit biases of Adam and Muon are to blame, but that constraining the update (e.g. via LoRA) narrows the gap between Muon and AdamW fine-tuning. Muon's increasing adoption makes this question more important than ever, and so in this paper we evaluate on both AdamW- and Muon-pretrained base models.

\paragraph{Geometry-aware and Muon-based LoRA optimizers}
The most closely related works are~\citet{cesista2026lora} and~\citet{bogachev2025lora}; these both attempt to approximate a low-rank solution to the spectral-regularized gradient descent problem. LoRA-Muon~\citep{cesista2026lora} is derived by approximating the spectral steepest-descent optimization problem using the triangle inequality applied to the spectral norm regularizer, and then minimizing to first order. We similarly linearize (we expect second-order terms to be suitably small), but we avoid the triangle inequality, as this is not necessarily tight under the spectral norm.
For their `Riemannion' optimizer,~\citet{bogachev2025lora} present the problem as Riemannian optimization on a constrained low-rank manifold. While some key ideas of this method are similar to our method (projecting onto the low-rank space), Riemannian optimization also involves a `retract' step. One issue with their retraction is requiring a truncated SVD, which cannot be efficiently replaced with an efficient matmul-only approximation. In practice this leads to Riemannion being slower than alternatives, especially for larger LoRA rank.
By contrast, we obtain the step using least squares on the orthogonalized gradient direction, avoiding expensive manifold constraints that don't appear to benefit model performance.
The PoLoRA~\citep{ghosh2026polora} method is again similar to our work, but can be understood as a version of LoRA-Muon with additional preconditioning.
LoRA-RITE~\citep{yen2024lora} and Riemannian-preconditioned LoRA~\citep{zhang2024riemannian} attempt to keep the adapters invariant to parameterization. Another popular approach is to simply apply Muon to $B$ and $A$ independently~\citep{qu2026can,janson2026stabilizing,mo2025parameter}. This ignores the geometry of the linear layer and causes the update to depend on the adapter parameterization, whereas our update is first-order parameterization-invariant (see App.~\ref{sec:transport}). Appendix~\ref{app:comparison} has a more detailed comparison of the various approaches.

\paragraph{Low-rank optimization beyond LoRA fine-tuning}
Some works, such as Spectron~\citep{janson2026stabilizing}, LORO~\citep{mo2025parameter}, GaLore~\citep{zhao2024galore}, and LoRA-Pre~\citep{wang2026taming} target efficient low-rank pretraining, either through stabilization, or optimizer state compression. \citet{he2025low} accelerate Muon itself with a low-rank orthogonalization routine that exploits the low-rank structure of the gradient, but operate on full dense weight matrices rather than in the LoRA adapter setting. In contrast to all of these, our goal is a simple, principled, memory-efficient Muon update for the composite LoRA adapter $\Delta W = BA$, rather than the full base weight $W_0$.

\section{Background}
\subsection{Muon}
Given a gradient matrix $G_W \in \mathbb R^{d_1 \times d_2}$ for weight matrix $W \in \mathbb R^{d_1 \times d_2}$, Muon selects the step direction $\delta W^\star \in \mathbb R^{d_1 \times d_2}$
through a regularized minimization under the spectral norm $\|\cdot\|_2$ (largest singular value)~\citep{bernstein2024old}:
\begin{equation}\label{eq:relaxed_muon}
  \delta W^\star =\argmin_{\delta W} \; \left[\langle G_W, \delta  W \rangle + \frac{\lambda}{2} \|\delta W\|_2^2\right].
\end{equation}
The solution, assuming $G_W$ is full rank, is proportional to $-\msign(G_W)$, where $\msign(X) = UV^\top$ for singular value decomposition $X = U\Sigma V^\top$, thereby orthogonalizing the gradient (or, in practice, the momentum). The solution is not unique for low-rank $G_W$, though we use the convention that $\msign(X) = U\operatorname{sign}(\Sigma) V^\top$, so that non-zero singular values are promoted, and the rest remain equal to zero.
\subsection{LoRA}
LoRA~\citep{hu2022lora} is a parameter-efficient method for fine-tuning neural networks.
LoRA adds adapters in the form $\Delta W = BA$ to the pretrained base weight $W_0$ of a layer, with $B \in \mathbb{R}^{d_1 \times r}$ and $A \in \mathbb{R}^{r \times d_2}$ where $r \ll d_1, d_2$. Typically $B$ is initialized to zero to ensure that, initially, $\Delta W = 0$. The factors $B,A$ are usually chosen by minimizing the fine-tuning loss with the AdamW optimizer. We instead seek to find the analogue to Muon's whitened update in the LoRA setting.
\section{Method}
Let $\Delta W = BA$ be the adapter with gradient $G_{\Delta W}$.
Our goal is to find updates $\delta B, \delta A$ for the factors $B,A$ such that the implied update $\delta \Delta W$ to the adapter $\Delta W$ has an approximately flat spectrum. We emphasize that simply orthogonalizing the factor updates $\delta B, \delta A$ does not imply that $\delta \Delta W$ itself is orthogonal. Under spectral steepest descent, the optimal adapter update $\delta \Delta W^\star$ is given by
\begin{equation}\label{eq:adapter_muon_full}
  \delta \Delta W^\star = \argmin_{\delta \Delta W}\ \big[\,\big\langle G_{\Delta W},\, \delta \Delta W\big\rangle
  + \tfrac{\lambda}{2}\big\|\delta \Delta W\big\|_2^2\,\big].
\end{equation}
If $r = \min(d_1,\,d_2)$ then $\Delta W$ is full rank and
this is an identical problem to Eq.~\eqref{eq:relaxed_muon}, which is solved by whitening the gradient~\citep{bernstein2024old}, i.e.,
\begin{equation}\label{eq:ideal_full_gradient}
    \delta \Delta W^\star = -\eta\msign(G_{\Delta W}),
\end{equation}
where proportionality constants have been absorbed into $\eta$.
However, we do not get the same solution under a low-rank constraint $r < \min(d_1, d_2)$. We have,
\begin{equation}
\delta \Delta W = (B + \delta B)(A + \delta A) - BA = \delta B\,A + B\,\delta A + \delta B\,\delta A ,
\end{equation}
but generally it is not possible to find $\delta B, \,\delta A$ that satisfy $ \delta \Delta W = -\eta \msign(G_{\Delta W})$.
The problem here is that the change $\delta \Delta W$ has rank at most $2r$, but $\msign(G_{\Delta W})$ is usually full rank.

We proceed by writing Eq.~\eqref{eq:adapter_muon_full} in terms of the adapter matrices,
\begin{equation}\label{eq:lora_exact_muon}
  (\delta B^\star, \delta A^\star) = \argmin_{\delta B, \delta A}\ \bigg[ \big\langle G_{\Delta W},\, \delta B\,A + B\,\delta A + \delta B\,\delta A\big\rangle
  + \tfrac{\lambda}{2}\big\|\delta B\,A + B\,\delta A + \delta B\,\delta A\big\|_2^2\bigg].
\end{equation}
This is a non-linear optimization problem due to the second-order terms, and it has no known closed-form solution.
For small updates $\delta B, \delta A$, we approximate by dropping these second-order terms, giving
\begin{equation}\label{eq:lora_param_muon}
  (\delta B^\star, \delta A^\star) = \argmin_{\delta B, \delta A}\ \bigg[\big\langle G_{\Delta W},\, \delta B\,A + B\,\delta A\big\rangle
  + \tfrac{\lambda}{2}\big\|\delta B\,A + B\,\delta A\big\|_2^2\bigg].
\end{equation}
Since it is generally not possible to achieve the rank-unconstrained minimum, we instead try to get as close as possible. We solve a least-squares problem for the linearized update map (overloading the notation $\delta B^\star,\, \delta A^\star$ for convenience),
\begin{equation}\label{eq:ls}
  (\delta B^\star, \delta A^\star) = \argmin_{\delta B,\,\delta A}\ \big\|\,\delta\Delta W^\star - (\delta B\,A + B\,\delta A)\,\big\|_F^2.
\end{equation}
Taking gradients of the objective w.r.t. $\delta B$ and $\delta A$ and setting to zero, we have,
\begin{subequations}
\begin{align}
0 &= -2(\delta \Delta W^\star - \delta B\, A - B\,\delta A)A^\top\label{eq:ddb},\\
0 &= -2B^\top(\delta \Delta W^\star - \delta B\, A - B\,\delta A)\label{eq:dda}.
\end{align}
\end{subequations}
Assuming $B$ and $A$ have full rank, right-multiplying Eq.~\eqref{eq:ddb} by $(AA^\top)^{-1}$ and left-multiplying Eq.~\eqref{eq:dda} by $(B^\top B)^{-1}$ gives
\begin{subequations}\label{eq:da_db_rearr}
\begin{alignat}{2}
\delta \Delta W^\star A^\dagger &= \delta B\, A A^\dagger+ B\,\delta A\,A^\dagger & &= \delta B + B \delta A\, A^\dagger,\label{eq:intermediatea}\\
B^\dagger\delta \Delta W^\star &= B^\dagger \delta B\, A + B^\dagger B\,\delta A & &= B^\dagger \delta B\, A + \delta A,\label{eq:intermediateb}
\end{alignat}
\end{subequations}
where $B^\dagger = (B^\top B)^{-1}B^\top $ and $A^\dagger = A^\top (AA^\top)^{-1}$ are the Moore-Penrose inverses satisfying $B^\dagger B = I = AA^\dagger$.
Note that the minimizer is not unique: if $(\delta B, \delta A)$ is a minimizer of Eq.~\eqref{eq:ls}, then so is $(\delta B + BS, \delta A - SA)$ since,
\begin{align}
\delta B \, A + B\,\delta A  = (\delta B  +BS) A + B\,(\delta A  - SA).
\end{align}
Rearranging Eq.~\eqref{eq:intermediateb},
and setting $S = -B^\dagger \delta B$ means that 
\begin{subequations}\label{eq:factors}
\begin{align}
\delta A^\star &= B^\dagger \delta \Delta W^\star\label{eq:factorsa}
\intertext{is one particular minimizer, and then substituting into Eq.~\eqref{eq:intermediatea} gives the corresponding $\delta B^\star$,}
\delta B^\star &= (I - BB^\dagger)\delta \Delta W^\star A^\dagger. \label{eq:factorsb}
\end{align}
\end{subequations}
\subsection{Low-rank gradients}\label{sec:efficient_compute}
While the solution $(\delta B^\star,\, \delta A^\star)$ given by Eq.~\eqref{eq:factors} is simple, it is not efficiently computed. The $\delta \Delta W^\star = -\eta \msign (G_{\Delta W})$ term requires orthogonalization of a large matrix (remember that $G_{\Delta W}\in \mathbb{R}^{d_1\times d_2}$), which we would like to avoid. Fortunately, we do not actually require the full gradient $G_{\Delta W}$, and can instead use the low-rank gradient,
\begin{equation}\label{eq:H}
  H := G_{\Delta W} - (I-P_B)\,G_{\Delta W}\,(I-P_A) = P_B G_{\Delta W} + (I - P_B)G_{\Delta W} P_A,
\end{equation}
where $P_B = BB^\dagger$ and $P_A = A^\dagger A$ are projections onto the column space of $B$ and row space of $A$ respectively. The low-rank gradient $H$ is the projection of the gradient onto the subspace of weight updates that the current LoRA adapter can express.
The two gradients $H$ and $G_{\Delta W}$ are first-order equivalent, in the sense that for any $\delta B, \delta A$, 
\begin{equation}
    \langle G_{\Delta W},\, \delta B\, A + B \delta A \rangle = \langle H, \delta B\, A + B \delta A \rangle.
\end{equation}
Thus when linearizing the Muon objective, we can replace $G_{\Delta W}$ with $H$ to obtain,
\begin{align}
(\delta B, \delta A) = \argmin_{\delta B, \delta A}\left[\langle H,\,\delta B\,A + B\,\delta A \rangle +  \tfrac{\lambda}{2}\|\delta B\,A + B\,\delta A\|_2^2\right].
\end{align}
In particular, this means we use $\delta\Delta W^\star = -\eta\msign(H)$. We show in the next section how this can be exploited to compute the update efficiently.


\renewcommand{\CommentSty}[1]{\textcolor{gray}{\texttt{#1}}}
\begin{figure}[t]
\begin{minipage}[t]{0.45\textwidth}
\begin{algorithm}[H]
\caption{sMuon}
\label{alg:conceptual}
\KwIn{$B \in \mathbb{R}^{d_1 \times r}$, $A \in \mathbb{R}^{r \times d_2}$; momenta $M_B, M_A$; lr $\eta$; weight decay $\lambda$}

\DontPrintSemicolon

\tcp{Low-rank momentum}
$H \leftarrow (B^\dagger)^\top M_A + (I - BB^\dagger)\, M_B\, (A^\dagger)^\top$\;

\tcp{Compute step}
$\delta A \leftarrow B^\dagger\, \msign(H)$\;
$\delta B \leftarrow (I - BB^\dagger)\, \msign(H)\, A^\dagger$\;
\tcp{Update parameters (split wd)}
$s \leftarrow \sqrt{1 - \lambda\eta}$\;
$B \leftarrow s\,B - \frac{\eta}{s}\,\delta B$\;
$A \leftarrow s\,A - \frac{\eta}{s}\,\delta A$\;

\tcp{Project momentum}
$M_B \leftarrow H\,A^\top$\;
$M_A \leftarrow B^\top H$\;
\end{algorithm}
\end{minipage}\hspace{0.75cm}
\begin{minipage}[t]{0.49\textwidth}
\begin{algorithm}[H]
\caption{Efficient/matmul-only sMuon}
\label{alg:matmul_only}
\KwIn{Same as Alg.~\ref{alg:conceptual}; jitter $\epsilon > 0$}

\DontPrintSemicolon
$S_B \leftarrow \mathtt{invroot}(B^\top B + \epsilon I)$ \; $S_A \leftarrow \mathtt{invroot}(AA^\top + \epsilon I)$\;
\tcp{Factor $H = BX + YA$}
$X \leftarrow S_B^2 M_A$ \; $Y \leftarrow M_B S_A^2 - B S_B^2(B^\top M_B S_A^2)$\;

\tcp{Bases for $\mathrm{col}(H)$, $\mathrm{row}(H)$}
$U_2 \leftarrow \msign(Y)$\;
$V_1 \leftarrow \msign(A^\top)$ \;
$V_2 \leftarrow \msign(X^\top\! - V_1(V_1^\top X^\top))$\;

\tcp{$2r \times 2r$ core}
$\Omega \leftarrow \msign\!\left(\begin{bmatrix}S_B M_A V_1 & S_B M_A V_2\\ (U_2^\top Y)(AV_1) & 0\end{bmatrix}\right)$\;

\tcp{Compute step}
$\delta A \leftarrow S_B(\Omega_{11}V_1^\top + \Omega_{12}V_2^\top)$\;
$\delta B \leftarrow U_2\,\Omega_{21} S_A$\;
\tcp{Update parameters}
$s \leftarrow \sqrt{1-\lambda\eta}$\;
$B' \leftarrow sB - \tfrac{\eta}{s}\delta B$, \; $A' \leftarrow sA - \tfrac{\eta}{s}\delta A$\;

\tcp{Project momentum}
$M_A \leftarrow ({B'}^\top B)X + ({B'}^\top Y)A$\;
$M_B \leftarrow B(X{A'}^\top) + Y(A{A'}^\top)$\;
$B \leftarrow B'$,  $A \leftarrow A'$\;
\end{algorithm}
\end{minipage}
\hspace{-1.5cm}
\label{fig:algorithms}
\end{figure}

\subsection{Efficient, matmul-only computation}
We now show how to compute the descent steps $\delta B^\star$ and $\delta A^\star$ efficiently. One of the principal advantages of our method versus Riemannion is that it permits a matmul-only solution --- Riemannion requires truncated SVDs during the retraction step, which makes parallelization/batching of computation slower (see Table~\ref{tab:speed}). In particular, the most expensive operations required are $d\times r$ matmuls, which are easily optimized with deep learning libraries.

The gradients for $B$ and $A$ are $G_B= G_{\Delta W} A^\top$ and $G_A = B^\top G_{\Delta W}$ respectively; the low-rank gradient of Eq.~\eqref{eq:H} can therefore be written as,
\begin{equation}\label{eq:H_factor_main}
  H = B\underbrace{(B^\top B)^{-1} G_A}_{=:\,X} \;+\; \underbrace{(I - P_B)\,G_B\,(AA^\top)^{-1}}_{=:\,Y}\, A,
\end{equation}
where $X\in\mathbb{R}^{r\times d_2}$ and
$Y\in\mathbb{R}^{d_1\times r}$. Assuming $B,\,A,\,X,\,Y$ have rank $r$,
we have $\operatorname{rank}(H)\le 2r$.

The key identity we exploit is as follows.
Let $U \in\mathbb{R}^{d_1\times 2r},\,V\in\mathbb{R}^{d_2\times 2r}$ be any matrices with orthonormal columns, such that the column space of $U$ contains the column space of $H$ and the column space of $V$ contains the row space of $H$, then\footnote{To see why Eq.~\eqref{eq:msign_lift} is true, note that since $UU^\top$ and $VV^\top$ act as the identity on the column and row spaces of $H$, we have $H = U(U^\top H V)V^\top$. Let $P, \Sigma,  Q^\top =\mathrm{SVD}(U^\top H V)$. Then $H = (UP)\Sigma(VQ)^\top$ is a valid, albeit truncated, SVD of $H$ ($UP$ and $VQ$ have orthonormal columns), so $\msign(H) = (UP)(VQ)^\top = U(PQ^\top)V^\top = U\,\msign(U^\top H V)\,V^\top$.},
\begin{equation}\label{eq:msign_lift}
  \msign(H)
  =
  U\,\msign(U^\top H V)\,V^\top.
\end{equation}
Eq.~\eqref{eq:msign_lift} lets us convert a $d_1\times d_2$ orthogonalization into a $2r\times 2r$ orthogonalization, and we will show shortly that instantiating $H \in \mathbb R^{d_1 \times d_2}$ can be avoided. To choose $U$ and $V$, note that
since $H=BX+YA$, its column space satisfies $\operatorname{col}(H) \subseteq \operatorname{col}(B)+\operatorname{col}(Y)$, while its row space satisfies $\operatorname{row}(H) \subseteq \operatorname{row}(A)+\operatorname{row}(X)$.
We therefore construct orthonormal bases for these two spaces as,
\begin{equation}
  U=\begin{pmatrix}U_1&U_2\end{pmatrix},
  \qquad
  V=\begin{pmatrix}V_1&V_2\end{pmatrix},
\end{equation}
with each block a polar factor (and therefore orthonormal columns) in the appropriate space,
\begin{subequations}\label{eq:polar_defs}
\begin{align}
  U_1 &= B(B^\top B)^{-1/2},
  &
  U_2 &= Y(Y^\top Y)^{-1/2},\\
  V_1 &= A^\top(AA^\top)^{-1/2},
  &
  V_2 &= Z(Z^\top Z)^{-1/2}, &\text{where } Z = (I-P_A)X^\top.
\end{align}
\end{subequations}
We have
$U_1^\top U_2=0$ since $B^\top Y =B^\top (I - P_B)G_B(AA^\top)^{-1} =0$, and $V_1^\top V_2=0$ since
$AZ=A(I - P_A)X^\top=0$. So $U$ and $V$ have orthonormal columns, satisfying the conditions for Eq.~\eqref{eq:msign_lift}.

Let $S_B=(B^\top B)^{-1/2}$ and 
$S_A=(AA^\top)^{-1/2}$. We calculate $U^\top H V$ efficiently,
\begin{align}\label{eq:cheap_uhv}
U^\top H V
&=
\begin{pmatrix}
U_1^\top B & U_1^\top Y\\
U_2^\top B & U_2^\top Y
\end{pmatrix}
\begin{pmatrix}
XV_1 & XV_2\\
AV_1 & AV_2
\end{pmatrix}
\notag\\
&=
\begin{pmatrix}
S_B^{-1} & 0\\
0 & U_2^\top Y
\end{pmatrix}
\begin{pmatrix}
XV_1 & XV_2\\
AV_1 & 0
\end{pmatrix}
\notag\\
&=
\begin{pmatrix}
S_B G_A V_1 & S_B G_A V_2\\
(U_2^\top Y)(AV_1) & 0
\end{pmatrix},
\end{align}
which avoids instantiating $H$ directly. Writing its orthogonalization in blocks of size $r \times r$,
\begin{align}
\Omega:=\msign(U^\top H V)
=
\begin{pmatrix}
\Omega_{11} & \Omega_{12}\\
\Omega_{21} & \Omega_{22}
\end{pmatrix},
\end{align}
we now calculate the final factor updates $\delta B^\star$ and $\delta A^\star$. Substituting $\msign(H)=U\Omega V^\top$
into the previous expressions for $\delta B^\star$ and $\delta A^\star$ (Eq.~\ref{eq:factors}) gives,
\begin{subequations}
\begin{align}
\delta A^\star =-\eta B^\dagger U\Omega V^\top
=
-\eta S_B\left(\Omega_{11}V_1^\top+\Omega_{12}V_2^\top\right),
\end{align}
where we used
$B^\dagger U_1=S_B$ and $B^\dagger U_2=0$.
Similarly,
\begin{align}
\delta B^\star
=
-\eta (I-P_B)U\Omega V^\top A^\dagger
=
-\eta U_2\,\Omega_{21}\,S_A,
\end{align}
\end{subequations}
using
$(I-P_B)U_1=0$, $(I-P_B)U_2=U_2$,
$V_1^\top A^\dagger=S_A$, and $V_2^\top A^\dagger=0$.

All components of $\delta B^\star$ and $\delta A^\star$ can be computed with matmuls, matrix inverse roots, and $\msign$. In practice, both the $r \times r$ inverse square roots and the $\msign$'s can be evaluated with matrix polynomial iterations~\citep{amsel2025polar,gramns2025}, which are themselves sequences of matrix products.
\begin{table}[t]
\centering
\caption{Per-eval accuracy (\%), grouped by base model. Each column is an optimizer (best run by validation loss); per row the best optimizer is bolded and cells are shaded green (best) to white (worst).}
\label{tab:results-all}
\hspace*{-1cm}
\begin{tabular}{lccccc}
\toprule
Eval & AdamW & Per-factor Muon & LoRA-Muon & Riemannion & sMuon (ours) \\
\midrule
\multicolumn{6}{l}{\textbf{Qwen2.5-3B}} \\
Piqa & \cellcolor[HTML]{FBFEFA} 81.7 & \cellcolor[HTML]{E3F4DF} 82.1 & \cellcolor[HTML]{FBFEFA} 81.7 & \cellcolor[HTML]{B9E1BA} \textbf{82.7} & \cellcolor[HTML]{FBFEFA} 81.7 \\
Winogrande & \cellcolor[HTML]{B9E1BA} \textbf{62.3} & \cellcolor[HTML]{F2FAEF} 61.2 & \cellcolor[HTML]{E3F4DF} 61.4 & \cellcolor[HTML]{FBFEFA} 61.0 & \cellcolor[HTML]{D0ECCD} 61.5 \\
OBQA & \cellcolor[HTML]{D0ECCD} 78.4 & \cellcolor[HTML]{B9E1BA} \textbf{79.2} & \cellcolor[HTML]{FBFEFA} 76.0 & \cellcolor[HTML]{E3F4DF} 77.4 & \cellcolor[HTML]{F2FAEF} 77.0 \\
Hellaswag & \cellcolor[HTML]{B9E1BA} \textbf{81.7} & \cellcolor[HTML]{E3F4DF} 78.7 & \cellcolor[HTML]{FBFEFA} 76.9 & \cellcolor[HTML]{F2FAEF} 77.6 & \cellcolor[HTML]{D0ECCD} 80.3 \\
BoolQ & \cellcolor[HTML]{B9E1BA} \textbf{67.7} & \cellcolor[HTML]{EDF8EA} 66.5 & \cellcolor[HTML]{FBFEFA} 65.6 & \cellcolor[HTML]{B9E1BA} \textbf{67.7} & \cellcolor[HTML]{D6EFD3} 67.4 \\
Arc-C & \cellcolor[HTML]{E3F4DF} 83.3 & \cellcolor[HTML]{B9E1BA} \textbf{83.6} & \cellcolor[HTML]{FBFEFA} 82.9 & \cellcolor[HTML]{E3F4DF} 83.3 & \cellcolor[HTML]{B9E1BA} \textbf{83.6} \\
MMLU & \cellcolor[HTML]{FBFEFA} 64.2 & \cellcolor[HTML]{E3F4DF} 64.9 & \cellcolor[HTML]{B9E1BA} \textbf{65.0} & \cellcolor[HTML]{E3F4DF} 64.9 & \cellcolor[HTML]{E3F4DF} 64.9 \\
\midrule
\multicolumn{6}{l}{\textbf{Llama-3.2-3B}} \\
Piqa & \cellcolor[HTML]{B9E1BA} \textbf{77.9} & \cellcolor[HTML]{D0ECCD} 77.1 & \cellcolor[HTML]{FBFEFA} 74.2 & \cellcolor[HTML]{E3F4DF} 76.9 & \cellcolor[HTML]{F2FAEF} 76.6 \\
Winogrande & \cellcolor[HTML]{FBFEFA} 51.0 & \cellcolor[HTML]{F2FAEF} 52.0 & \cellcolor[HTML]{D0ECCD} 53.8 & \cellcolor[HTML]{B9E1BA} \textbf{54.9} & \cellcolor[HTML]{E3F4DF} 52.4 \\
OBQA & \cellcolor[HTML]{E3F4DF} 67.2 & \cellcolor[HTML]{D0ECCD} 68.4 & \cellcolor[HTML]{FBFEFA} 63.4 & \cellcolor[HTML]{B9E1BA} \textbf{69.4} & \cellcolor[HTML]{F2FAEF} 66.8 \\
Hellaswag & \cellcolor[HTML]{E3F4DF} 63.7 & \cellcolor[HTML]{B9E1BA} \textbf{64.9} & \cellcolor[HTML]{FBFEFA} 58.9 & \cellcolor[HTML]{D0ECCD} 64.3 & \cellcolor[HTML]{F2FAEF} 63.3 \\
BoolQ & \cellcolor[HTML]{FBFEFA} 63.1 & \cellcolor[HTML]{E3F4DF} 63.4 & \cellcolor[HTML]{D0ECCD} 63.8 & \cellcolor[HTML]{B9E1BA} \textbf{64.1} & \cellcolor[HTML]{F2FAEF} 63.3 \\
Arc-C & \cellcolor[HTML]{B9E1BA} \textbf{67.9} & \cellcolor[HTML]{D0ECCD} 66.2 & \cellcolor[HTML]{FBFEFA} 60.5 & \cellcolor[HTML]{E3F4DF} 64.9 & \cellcolor[HTML]{F2FAEF} 63.2 \\
MMLU & \cellcolor[HTML]{EDF8EA} 54.4 & \cellcolor[HTML]{D6EFD3} 54.8 & \cellcolor[HTML]{FBFEFA} 52.7 & \cellcolor[HTML]{B9E1BA} \textbf{54.9} & \cellcolor[HTML]{EDF8EA} 54.4 \\
\midrule
\multicolumn{6}{l}{\textbf{DeepSeek-V2-Lite}} \\
Piqa & \cellcolor[HTML]{B9E1BA} \textbf{79.2} & \cellcolor[HTML]{EDF8EA} 76.7 & \cellcolor[HTML]{D6EFD3} 78.2 & \cellcolor[HTML]{FBFEFA} 75.5 & \cellcolor[HTML]{B9E1BA} \textbf{79.2} \\
Winogrande & \cellcolor[HTML]{EDF8EA} 58.2 & \cellcolor[HTML]{D6EFD3} 58.8 & \cellcolor[HTML]{B9E1BA} \textbf{59.7} & \cellcolor[HTML]{FBFEFA} 55.8 & \cellcolor[HTML]{EDF8EA} 58.2 \\
OBQA & \cellcolor[HTML]{EDF8EA} 68.0 & \cellcolor[HTML]{B9E1BA} \textbf{70.8} & \cellcolor[HTML]{D6EFD3} 69.0 & \cellcolor[HTML]{FBFEFA} 61.2 & \cellcolor[HTML]{B9E1BA} \textbf{70.8} \\
Hellaswag & \cellcolor[HTML]{D0ECCD} 74.4 & \cellcolor[HTML]{E3F4DF} 71.3 & \cellcolor[HTML]{F2FAEF} 71.2 & \cellcolor[HTML]{FBFEFA} 65.4 & \cellcolor[HTML]{B9E1BA} \textbf{76.4} \\
BoolQ & \cellcolor[HTML]{EDF8EA} 64.6 & \cellcolor[HTML]{FBFEFA} 64.1 & \cellcolor[HTML]{EDF8EA} 64.6 & \cellcolor[HTML]{D6EFD3} 64.7 & \cellcolor[HTML]{B9E1BA} \textbf{65.2} \\
Arc-C & \cellcolor[HTML]{B9E1BA} \textbf{70.6} & \cellcolor[HTML]{FBFEFA} 64.9 & \cellcolor[HTML]{B9E1BA} \textbf{70.6} & \cellcolor[HTML]{EDF8EA} 67.2 & \cellcolor[HTML]{D6EFD3} 67.6 \\
MMLU & \cellcolor[HTML]{B9E1BA} \textbf{55.4} & \cellcolor[HTML]{D0ECCD} 54.7 & \cellcolor[HTML]{FBFEFA} 53.1 & \cellcolor[HTML]{E3F4DF} 54.1 & \cellcolor[HTML]{F2FAEF} 53.5 \\
\midrule
\multicolumn{6}{l}{\textbf{Moonlight-16B-A3B}} \\
Piqa & \cellcolor[HTML]{E3F4DF} 84.7 & \cellcolor[HTML]{F2FAEF} 84.5 & \cellcolor[HTML]{D0ECCD} 85.6 & \cellcolor[HTML]{FBFEFA} 84.4 & \cellcolor[HTML]{B9E1BA} \textbf{86.2} \\
Winogrande & \cellcolor[HTML]{FBFEFA} 59.4 & \cellcolor[HTML]{F2FAEF} 61.2 & \cellcolor[HTML]{E3F4DF} 63.3 & \cellcolor[HTML]{D0ECCD} 64.0 & \cellcolor[HTML]{B9E1BA} \textbf{65.8} \\
OBQA & \cellcolor[HTML]{EDF8EA} 81.0 & \cellcolor[HTML]{D6EFD3} 81.2 & \cellcolor[HTML]{D6EFD3} 81.2 & \cellcolor[HTML]{FBFEFA} 78.6 & \cellcolor[HTML]{B9E1BA} \textbf{82.4} \\
Hellaswag & \cellcolor[HTML]{FBFEFA} 77.9 & \cellcolor[HTML]{F2FAEF} 80.5 & \cellcolor[HTML]{D0ECCD} 84.7 & \cellcolor[HTML]{E3F4DF} 82.3 & \cellcolor[HTML]{B9E1BA} \textbf{85.9} \\
BoolQ & \cellcolor[HTML]{F2FAEF} 68.4 & \cellcolor[HTML]{FBFEFA} 67.1 & \cellcolor[HTML]{B9E1BA} \textbf{69.8} & \cellcolor[HTML]{E3F4DF} 68.5 & \cellcolor[HTML]{D0ECCD} 69.5 \\
Arc-C & \cellcolor[HTML]{D0ECCD} 85.6 & \cellcolor[HTML]{FBFEFA} 81.9 & \cellcolor[HTML]{E3F4DF} 84.3 & \cellcolor[HTML]{F2FAEF} 83.3 & \cellcolor[HTML]{B9E1BA} \textbf{87.0} \\
MMLU & \cellcolor[HTML]{D0ECCD} 66.1 & \cellcolor[HTML]{E3F4DF} 65.9 & \cellcolor[HTML]{F2FAEF} 65.6 & \cellcolor[HTML]{B9E1BA} \textbf{66.2} & \cellcolor[HTML]{FBFEFA} 64.5 \\
\bottomrule
\end{tabular}
\end{table}

\subsection{Algorithm}
We combine the above updates with split weight decay~\citep{cesista2026lora} to match the weight decay dynamics of full-parameter training to first order, and momentum projection (inspired by momentum transport of~\cite{bogachev2025lora}, see App.~\ref{sec:transport} for details) to ensure reparameterization invariance, yielding the sMuon optimizer for LoRA adapters. We provide Alg.~\ref{alg:conceptual} as a pedagogical version, while Alg.~\ref{alg:matmul_only} incorporates the efficient matmul-only computations from the previous section.

We also discuss a more numerically stable variant in Appendix~\ref{app:stable}. One cause for concern is the practice of initializing the adapter $\Delta W=BA = 0$, and then computing inverse roots of $B^\top B$ and $AA^\top$.  We found that, despite initializing $\Delta W = 0$, calculating inverse roots of $AA^\top$ and $B^\top B$ was not an issue when using jitter and a semi-orthonormal $B$, $A=0$ initialization.
\FloatBarrier
\section{Experiments}
\begin{table}[t]
\centering
\caption{Code pass@1 (\%), grouped by base model. Each column is an optimizer (best run by validation loss); per row the best optimizer is bolded and cells are shaded green (best) to white (worst).}
\label{tab:results-all-code}
\hspace*{-1cm}
\begin{tabular}{lccccc}
\toprule
Eval & AdamW & Per-factor Muon & LoRA-Muon & Riemannion & sMuon (ours) \\
\midrule
\multicolumn{6}{l}{\textbf{Qwen2.5-3B}} \\
HumanEval & \cellcolor[HTML]{B9E1BA} \textbf{54.9} & \cellcolor[HTML]{FBFEFA} 20.1 & \cellcolor[HTML]{E3F4DF} 43.3 & \cellcolor[HTML]{F2FAEF} 31.1 & \cellcolor[HTML]{D0ECCD} 51.8 \\
HumanEval+ & \cellcolor[HTML]{B9E1BA} \textbf{48.8} & \cellcolor[HTML]{FBFEFA} 17.1 & \cellcolor[HTML]{E3F4DF} 39.0 & \cellcolor[HTML]{F2FAEF} 25.6 & \cellcolor[HTML]{D0ECCD} 43.9 \\
MBPP & \cellcolor[HTML]{EDF8EA} 65.9 & \cellcolor[HTML]{B9E1BA} \textbf{69.3} & \cellcolor[HTML]{D6EFD3} 68.0 & \cellcolor[HTML]{EDF8EA} 65.9 & \cellcolor[HTML]{FBFEFA} 62.7 \\
MBPP+ & \cellcolor[HTML]{E3F4DF} 57.1 & \cellcolor[HTML]{B9E1BA} \textbf{60.6} & \cellcolor[HTML]{D0ECCD} 58.7 & \cellcolor[HTML]{F2FAEF} 56.9 & \cellcolor[HTML]{FBFEFA} 53.4 \\
\midrule
\multicolumn{6}{l}{\textbf{Llama-3.2-3B}} \\
HumanEval & \cellcolor[HTML]{EDF8EA} 22.0 & \cellcolor[HTML]{FBFEFA} 18.3 & \cellcolor[HTML]{D6EFD3} 29.9 & \cellcolor[HTML]{B9E1BA} \textbf{30.5} & \cellcolor[HTML]{D6EFD3} 29.9 \\
HumanEval+ & \cellcolor[HTML]{F2FAEF} 19.5 & \cellcolor[HTML]{FBFEFA} 16.5 & \cellcolor[HTML]{E3F4DF} 25.0 & \cellcolor[HTML]{D0ECCD} 26.8 & \cellcolor[HTML]{B9E1BA} \textbf{27.4} \\
MBPP & \cellcolor[HTML]{E3F4DF} 43.1 & \cellcolor[HTML]{FBFEFA} 35.7 & \cellcolor[HTML]{F2FAEF} 42.1 & \cellcolor[HTML]{B9E1BA} \textbf{46.3} & \cellcolor[HTML]{D0ECCD} 45.2 \\
MBPP+ & \cellcolor[HTML]{E3F4DF} 35.4 & \cellcolor[HTML]{FBFEFA} 30.4 & \cellcolor[HTML]{F2FAEF} 34.1 & \cellcolor[HTML]{D0ECCD} 38.1 & \cellcolor[HTML]{B9E1BA} \textbf{38.6} \\
\midrule
\multicolumn{6}{l}{\textbf{DeepSeek-V2-Lite}} \\
HumanEval & \cellcolor[HTML]{B9E1BA} \textbf{36.6} & \cellcolor[HTML]{D0ECCD} 29.3 & \cellcolor[HTML]{FBFEFA} 22.0 & \cellcolor[HTML]{F2FAEF} 23.2 & \cellcolor[HTML]{E3F4DF} 26.8 \\
HumanEval+ & \cellcolor[HTML]{B9E1BA} \textbf{28.7} & \cellcolor[HTML]{D0ECCD} 23.8 & \cellcolor[HTML]{FBFEFA} 18.3 & \cellcolor[HTML]{F2FAEF} 19.5 & \cellcolor[HTML]{E3F4DF} 23.2 \\
MBPP & \cellcolor[HTML]{D0ECCD} 52.6 & \cellcolor[HTML]{B9E1BA} \textbf{53.7} & \cellcolor[HTML]{FBFEFA} 48.4 & \cellcolor[HTML]{F2FAEF} 49.5 & \cellcolor[HTML]{E3F4DF} 51.3 \\
MBPP+ & \cellcolor[HTML]{B9E1BA} \textbf{45.5} & \cellcolor[HTML]{D0ECCD} 45.0 & \cellcolor[HTML]{FBFEFA} 43.4 & \cellcolor[HTML]{F2FAEF} 43.7 & \cellcolor[HTML]{E3F4DF} 44.7 \\
\midrule
\multicolumn{6}{l}{\textbf{Moonlight-16B-A3B}} \\
HumanEval & \cellcolor[HTML]{F2FAEF} 49.4 & \cellcolor[HTML]{B9E1BA} \textbf{57.9} & \cellcolor[HTML]{E3F4DF} 53.7 & \cellcolor[HTML]{FBFEFA} 40.2 & \cellcolor[HTML]{D0ECCD} 54.9 \\
HumanEval+ & \cellcolor[HTML]{EDF8EA} 42.7 & \cellcolor[HTML]{B9E1BA} \textbf{48.8} & \cellcolor[HTML]{D6EFD3} 47.0 & \cellcolor[HTML]{FBFEFA} 35.4 & \cellcolor[HTML]{B9E1BA} \textbf{48.8} \\
MBPP & \cellcolor[HTML]{B9E1BA} \textbf{71.7} & \cellcolor[HTML]{E3F4DF} 71.2 & \cellcolor[HTML]{E3F4DF} 71.2 & \cellcolor[HTML]{FBFEFA} 68.8 & \cellcolor[HTML]{E3F4DF} 71.2 \\
MBPP+ & \cellcolor[HTML]{EDF8EA} 60.1 & \cellcolor[HTML]{EDF8EA} 60.1 & \cellcolor[HTML]{B9E1BA} \textbf{61.4} & \cellcolor[HTML]{FBFEFA} 58.7 & \cellcolor[HTML]{D6EFD3} 60.3 \\
\bottomrule
\end{tabular}
\end{table}

\begin{table}[t]
\centering
\caption{ReLoRA pretraining: final validation loss at each optimizer's optimal learning rate, where the optimal LR minimises the validation loss at the final logged step. Lower is better; the best optimizer is bolded and cells are shaded green (best) to white (worst).}
\label{tab:results-all-relora}
\hspace*{-1cm}
\begin{tabular}{lccccc}
\toprule
 & AdamW & Per-factor Muon & LoRA-Muon & Riemannion & sMuon (ours) \\
\midrule
Final val loss & \cellcolor[HTML]{FBFEFA} 3.83 & \cellcolor[HTML]{EDF8EA} 3.70 & \cellcolor[HTML]{D6EFD3} 3.68 & \cellcolor[HTML]{B9E1BA} \textbf{3.66} & \cellcolor[HTML]{B9E1BA} \textbf{3.66} \\
Optimal LR & 0.001 & 0.006 & 0.001 & 0.0006 & 0.0006 \\
\bottomrule
\end{tabular}
\end{table}

We compare sMuon to regular LoRA with AdamW, per-factor Muon (LoRA with Muon applied to each adapter weight $B$ and $A$ individually), LoRA-Muon~\citep{cesista2026lora}, and Riemannion~\citep{bogachev2025lora}. We evaluated on standard commonsense~\citep{clark2018think,clark2019boolq,zellers2019hellaswag,mihaylov2018openbookqa,bisk2020piqa,sakaguchi2021winogrande} and coding~\citep{chen2021humaneval,austin2021mbpp} benchmarks after training with SFT. We experimented with four models, Llama-3.2-3B~\citep{grattafiori2024llama}, Qwen2.5-3B~\citep{qwen2025qwen25}, DeepSeek-V2-Lite~\citep{deepseekai2024deepseekv2}, and Moonlight-16B-A3B~\citep{liu2025muon}. The first two models are dense models; the last two are sparse MoE models. The Moonlight model is of particular note because it was pretrained with Muon, whereas the others are all AdamW-pretrained. We also tested the methods in a pretraining setting using ReLoRA~\citep{lialin2023relora} which periodically merges and refreshes the adapter. We tuned the learning rate for each method and model, with example learning rate sweeps for the commonsense tasks shown in Appendix~\ref{app:lr_sweep}.

While our results are both model- and eval-dependent, sMuon performs strongly overall, especially when combined with the Moonlight-16B-A3B base model, a Muon-pretrained model (Tables~\ref{tab:results-all} and~\ref{tab:results-all-code}). sMuon also achieved the joint-best validation loss in the ReLoRA pretraining setting (Table~\ref{tab:results-all-relora}). We attribute performance gains, particularly over AdamW and per-factor Muon, to more accurately approximating the Muon objective. 

\paragraph{Experiment details}
LoRA-Muon and Riemannion were chosen as baselines in addition to the simpler AdamW and per-factor Muon baselines as these are the most similar to our method. Unfortunately, we could not find implementations for LoRA-Muon or Riemannion, so we implemented these ourselves.
For Riemannion and our method, we use an $A=0$ and $B$ orthonormal initialization, and for the remaining methods we used a standard initialization ($B=0$, $A$ random Gaussian). We scale the learning rate to approximately match AdamW RMS for all Muon optimizers (for details see App.~\ref{app:match_adamw_rms}). For all experiments, we used weight decay $0.01$, momentum $0.9$, inserted LoRA adapters into all hidden layers (except the router for models with MoE). We used GH200 for all experiments except to produce Table~\ref{tab:speed} for which we used A100. We used bf16 to store weights and optimizer states, but used fp32 to compute the step for Riemannion, sMuon, and LoRA-Muon.

For the ReLoRA experiments \citep{lialin2023relora}, we started with the modded-nanogpt codebase at `record 8'~\citep{jordan2024moddednanogpt}, which trains a 162M parameter transformer from scratch (using features including QK-norm, ReLU$^2$ activations) on the FineWeb dataset~\citep{penedo2024fineweb}, and evaluates on a 10M token validation set. We initialize LoRA adapters with rank 64, and train for a total of 3000 steps; every 500 steps we merge the adapters into the base weights and reinitialize the adapters; we use a cosine decay learning rate schedule, though after each merge+reinitialization, we also reset the learning rate to zero and warm back up to the target learning rate over 50 steps (this is also known as a jagged cosine schedule). We swept learning rates in 3e-4, 6e-4, 1e-3, 3e-3, 6e-3, 1e-2.

For the HumanEval~\citep{chen2021humaneval} and MBPP~\citep{austin2021mbpp} coding evals, we fine-tuned for 3 epochs on completing Python functions from the CodeAlpaca dataset~\citep{codealpaca}. For the MMLU~\citep{hendrycks2021mmlu} eval, we fine-tuned on a 50/50 mixture of FLAN~\citep{longpre2023flan} and MMLU training subset. For the remaining evals, we trained on a mixture of FLAN  and CommonsenseQA~\citep{talmor2019commonsenseqa}, and training subset for each eval. All SFT evals were 0-shot. For SFT training runs, we used $\alpha=32$, $r=16$, gradient clipping (clipping at $1$), and we used roughly $2^{16}$ tokens per batch.

\paragraph{Muon-based LoRA is strongest on the Muon-pretrained model}
We originally hypothesized that Muon-based LoRA methods might perform better with Moonlight, since it is a Muon-pretrained model. We do find some evidence for this when aggregating over all benchmarks in Tables~\ref{tab:results-all} and~\ref{tab:results-all-code}: AdamW only achieves top performance on one of the eleven tasks, while sMuon is best for six. Some methods are markedly below others on certain tasks (e.g. per-factor Muon with Llama-3.2-3B on HumanEval, Table~\ref{tab:results-all-code}). This could be due to optimizer instabilities or noise in the benchmarks or selection procedure.

\paragraph{Benefits of Muon pretraining remain with ReLoRA}
The ReLoRA experiment (Table~\ref{tab:results-all-relora}) removes optimizer
mismatch entirely by training from scratch. The separation between
optimizer families is clear. Every Muon-based method achieves lower validation loss than LoRA+AdamW. One explanation for the strong Muon performance is amplification of the small but real signal in the long tail of the gradient's spectrum~\citep{wang2025muon}.

\begin{table}[t]
\centering
\caption{All Muon optimizers are slightly slower than AdamW, but Riemannion adds huge overhead at larger rank. We show mean time for LoRA adapter optimizer step (ms, 2 s.f.), as well as the proportion (\%) of the total step time. Figures were obtained while training a 12 layer GPT model with $d_\text{model}=768$. We ran these timings on an A100, calculating the mean time over 25 steps, with different ranks $r\in\{2^4, 2^5, 2^6, 2^7\}$, batch size 128, 4 gradient accumulation steps. }\label{tab:speed}
\hspace*{-1cm}
\setlength{\tabcolsep}{4pt}
\begin{tabular}{lcccc}
\toprule
Step time, ms (\% total) $\downarrow$ & $r=16$ & $r=32$ & $r=64$ & $r=128$ \\
\midrule
AdamW & $0.66$ \scriptsize($0.047$\%) & $0.59$ \scriptsize($0.041$\%) & $0.65$ \scriptsize($0.046$\%) & $0.78$ \scriptsize($0.054$\%) \\
Per-factor Muon & $1.4$ \scriptsize($0.10$\%) & $1.6$ \scriptsize($0.11$\%) & $2.0$ \scriptsize($0.14$\%) & $3.6$ \scriptsize($0.25$\%) \\
LoRA-Muon & $4.5$ \scriptsize($0.32$\%) & $4.8$ \scriptsize($0.33$\%) & $6.3$ \scriptsize($0.44$\%) & $10$ \scriptsize($0.70$\%) \\
Riemannion & $90$ \scriptsize($6.0$\%) & $300$ \scriptsize($17$\%) & $700$ \scriptsize($33$\%) & $1700$ \scriptsize($54$\%) \\
sMuon (ours) & $8.7$ \scriptsize($0.62$\%) & $8.9$ \scriptsize($0.62$\%) & $12$ \scriptsize($0.80$\%) & $20$ \scriptsize($1.4$\%) \\
\midrule
Fwd+bwd, ms & $1400$ & $1432$ & $1429$ & $1459$ \\
\bottomrule
\end{tabular}
\end{table}

\paragraph{Efficiency gain from the matmul-only implementation}
Table~\ref{tab:speed} compares step times for the different optimizers. AdamW is the fastest; the only computation required by AdamW is element-wise, meaning that data communication is likely to be the largest overhead. Per-factor Muon is the next fastest, because computation is dominated by two $\msign(\cdot)$ computations (that can be batched). Both LoRA-Muon and sMuon require more matmuls than per-factor Muon, but the compute time is still negligible compared to the forward+backward pass time per step of $\sim1.4$s.
On the other hand, computing the Riemannion step is much slower, adding $\sim 50\%$ overhead to the forward+backward pass when $r=64$, and even longer for $r=128$. The computation is slow because Riemannion requires thin QRs and $2r\times 2r$ SVDs, which add overhead for larger $r$. Note that the retraction step in Riemannion specifically requires truncated SVD, which cannot be replaced by a more efficient Newton-Schulz algorithm. We found that compiling the core computation (i.e. the matmuls), and batching the computation leads to big efficiency gains with all the Muon methods except Riemannion.
\section{Conclusion}
We considered optimizing low-rank adapters using an approximate form of Muon. The resulting step is available in closed form, is invariant to the adapter parameterization $(B,A)\mapsto(BS,S^{-1}A)$ and is cheap, costing $O(r^2(d_1+d_2)+r^3)$ time and $O(r\max(d_1,d_2))$ memory, using less optimizer state than LoRA+AdamW. Our implementation, sMuon, has empirically flatter spectra than per-factor Muon and out-performs other methods on the Moonlight Muon-pretrained model. In ReLoRA pretraining every Muon-based method beats AdamW, with sMuon and Riemannion the strongest. In general, by respecting the geometry of the full layer, rather than treating $B$ and $A$ as unrelated matrices, we often benefit, and there is little extra cost.
\FloatBarrier
\bibliographystyle{plainnat}
\bibliography{refs}
\FloatBarrier
\appendix
\newpage
\section{Other algorithmic details}\label{app:other_alg_details}
We discuss some pre-existing techniques that we incorporated into sMuon.
\subsection{Match AdamW RMS}\label{app:match_adamw_rms}
Inspired by~\citet{liu2025muon}, we adopt a learning rate adjustment. The full-parameter scaling is $0.2 \sqrt{\max(d_\text{in}, d_\text{out})}$, with the overall update, including weight decay, becoming
\begin{align}x \leftarrow x - \lambda \eta x - 0.2 \eta\sqrt{\max (d_\text{in}, d_\text{out})}\, \cdot\, \delta x.
\end{align}
 To match this in the LoRA setting, we used
\begin{align}
  x \leftarrow x - \lambda \eta x - 0.2 \eta\sqrt{d_\text{in}d_\text{out}/r}\, \cdot\, \delta x.
\end{align}
We loosely justify this heuristic: assume the weight update $\delta \Delta W$ is orthogonalized. The rank  of $\delta \Delta W$  is at most $2r$ (it can be written as the difference of two rank $r$ matrices), and therefore $\|\delta \Delta W\|_F^2\leq 2r$.
So $\|\delta\Delta W\|_\text{rms}\leq  \sqrt{2r/(d_\text{in}d_\text{out})}$. We absorb the $\sqrt{2}$ factor, and use a scale of $0.2\sqrt{d_\text{in}d_\text{out}/r}$.
\subsection{Split weight decay}\label{sec:wd}
If we apply weight decay to $B$ and $A$ separately, this gives different dynamics compared to full-parameter training. Concretely,
\begin{align*}
  A_{t+1} &= (1-\eta\lambda)A_t - \eta\, \delta A, \quad  B_{t+1} = (1-\eta \lambda)B_t - \eta\,\delta B\\
\implies \delta \Delta W &= \big[(1-\eta\lambda)B_t - \eta\,\delta B\big]\big[(1-\eta\lambda)A_t - \eta\,\delta A\big] - B_tA_t\\
&= -\big[1 - (1-\eta\lambda)^2\big]B_tA_t - \eta(1-\eta\lambda)\big(B_t\,\delta A + \delta B\,A_t\big) + O(\eta^2),
\end{align*}
which is not recognizable in terms of full-parameter dynamics.
We adopt the split weight decay of~\citet{cesista2026lora} to preserve the first-order weight decay dynamics. Setting $s = \sqrt{1 - \lambda\eta}$, LoRA factors update as,
\begin{align}
  A_{t+1} &= s\,A_t - \tfrac{\eta}{s}\,\delta A, \\
  B_{t+1} &= s\,B_t - \tfrac{\eta}{s}\,\delta B.
\end{align}
Expanding the resulting weight change $\delta \Delta W = B_{t+1}A_{t+1} - B_tA_t$ gives,
\begin{align}
  \delta \Delta W &= \left(s B_t - \tfrac{\eta}{s}\,\delta B\right)\left(s A_t - \tfrac{\eta}{s}\,\delta A\right) - B_t A_t \\
           &= s^2 B_t A_t - \eta(\delta B\, A_t + B_t\, \delta A) + \tfrac{\eta^2}{s^2}\,\delta B\,\delta A - B_t A_t \\
           &= (s^2 - 1) B_t A_t - \eta(\delta B\, A_t + B_t\, \delta A) + \tfrac{\eta^2}{s^2}\,\delta B\,\delta A \\
           &= \underbrace{-\lambda\eta\, \Delta W_t}_\text{decay term} - \underbrace{\eta(\delta B\, A_t + B_t\, \delta A)}_\text{gradient term} + O(\eta^2)
\end{align}
which we identify as having the dynamics of decoupled weight decay~\citep{loshchilov2017decoupled}.
\subsection{Invariance to parameterization}\label{sec:transport}
In practice, we orthogonalize a momentum rather than a raw gradient. The natural choice for momentum accumulation is $M_A \leftarrow \beta M_A + (1-\beta)G_A$ and $M_B \leftarrow \beta M_B + (1-\beta)G_B$, and to reconstruct the effective momentum by substituting them for $G_A,\, G_B$ in the equation factorizing $H$,
\begin{equation}\label{eq:H_recon}
  \bar H_t = (B_t^\dagger)^\top M_{A,t} + (I-P_{B_t})\,M_{B,t} (A_t^\dagger)^\top .
\end{equation}
Here, $G_B = G_{\Delta W} A^\top$ and $G_A = B^\top G_{\Delta W}$ are the gradients of the adapter matrices $B$ and $A$ respectively. 
Consider two training runs, $(B_t,\, A_t)$ and $(B_t', A_t')$, that differ only in their parameterizations at each step,
\begin{equation}\label{eq:two_runs}
  (B_t', A_t') = (B_tS_t,\; S_t^{-1}A_t).
\end{equation}
Since $\Delta W_t' = B_t'A_t' = B_tA_t = \Delta W_t$, we want the optimizer to take equivalent steps throughout training, regardless of $S_t$. For this, it is necessary that $\bar H_t' = \bar H_t$.

\paragraph{Without transporting momentum buffers} The adapter gradients in the primed run are $G_{A,k}' = (B_k')^\top G_{\Delta W,k} = S_k^\top G_{A,k}$ and $G_{B,k}' = G_{\Delta W,k} (A_k')^\top = G_{B,k}S_k^{-\top}$. Unrolling the EMA, we have,
\begin{equation}\label{eq:unrolled_ema}
  M_{A,t}' = \sum_{k\le t}(1-\beta)\beta^{t-k}\,S_k^\top G_{A,k}, \qquad
  M_{B,t}' = \sum_{k\le t}(1-\beta)\beta^{t-k}\,G_{B,k}S_k^{-\top}.
\end{equation}
Thus, the `primed' run's effective momentum is,
\begin{equation}\label{eq:mixed_frame_method}
  \bar H_t' = \sum_{k\le t}(1-\beta)\beta^{t-k}\Big[(B_t^\dagger)^\top\,S_t^{-\top}S_k^\top\,G_{A,k}
  \;+\;(I-P_{B_t})\,G_{B,k}\,S_k^{-\top}S_t^\top\,(A_t^\dagger)^\top\Big],
\end{equation}
which differs from $\bar H_t$ by the factors $S_t^{-\top}S_k^\top$ and $S_k^{-\top}S_t^\top$. These are the identity only when $S_k = S_t$, so in general $\bar H_t' \neq \bar H_t$.
\paragraph{With momentum transport} We fix the issue by re-deriving the buffers from $\bar H$ against the factors as they stand at the end of the step,
\begin{equation}\label{eq:transport_project}
  M_{A} \leftarrow B^\top \bar H, \qquad M_{B} \leftarrow \bar H A^\top .
\end{equation}
Geometrically, this projects $\bar H$ onto the space in which our updates, $\delta B\,A + B\,\delta A$, exist. 
The momenta at the next step are then,
\begin{align}\label{eq:transport_single_frame}
  M_{A,t+1} &= B_{t+1}^\top\big(\beta \bar H_t + (1-\beta)G_{\Delta W,t+1}\big),\\
  M_{B,t+1} & = \big(\beta \bar H_t + (1-\beta)G_{\Delta W,t+1}\big)A_{t+1}^\top.
\end{align}
Assuming $\bar H_t' = \bar H_t$, we have $M_{A,t+1}' = (B_{t+1}')^\top \big(\beta \bar H_t + (1-\beta)G_{\Delta W,t+1}\big) = S_{t+1}^\top M_{A,t+1}$ and similarly $M_{B,t+1}' = M_{B,t+1}S_{t+1}^{-\top}$. By further substitution,
\begin{align*}
  \bar H_{t+1}' &= \big(S_{t+1}^{-1}B_{t+1}^\dagger\big)^\top S_{t+1}^\top M_{A,t+1}
    + (I-P_{B_{t+1}})\,M_{B,t+1}S_{t+1}^{-\top}\big(A_{t+1}^\dagger S_{t+1}\big)^\top \nonumber\\
  &= (B_{t+1}^\dagger)^\top \underbrace{S_{t+1}^{-\top}S_{t+1}^\top}_{I} M_{A,t+1}
    + (I-P_{B_{t+1}})\,M_{B,t+1}\underbrace{S_{t+1}^{-\top}S_{t+1}^\top}_{I}(A_{t+1}^\dagger)^\top\nonumber\\
  \;&=\; \bar H_{t+1}.
\end{align*}
So long as we project $\bar H$ into the momentum buffers at every step, the two runs will orthogonalize the same matrix, and therefore the realized steps will agree.

\section{Alternative stable algorithm}\label{app:stable}
See Alg.~\ref{alg:matmul_only_stable}, a mathematically equivalent but more stable alternative to Alg.~\ref{alg:matmul_only} from the main text. In particular, we perform two orthogonalization steps~\citep{giraud2005rounding} to help preserve orthogonality to machine precision, and include relative jitter to help stability for larger $r$. Another difference between Alg.~\ref{alg:matmul_only} and~\ref{alg:matmul_only_stable} is that we use slightly different $U_2$, $V_2.$
Alg.~\ref{alg:matmul_only_stable} uses $U_2 = \msign(\tilde Y)$, where $Y = \tilde Y S_A^2$, and $V_2 = \msign(\tilde Z)$, where $Z = \tilde Z S_B^2$. This is possible because the property we need is $\operatorname{col}(Y) = \operatorname{col}(\tilde Y)$, and similar for $Z$/$\tilde Z$, which is true when $S_B$ and $S_A$ are invertible (we assume this to be the case).

One can also consider additional interventions, such as replacing $\msign(X) = X(X^\top X)^{-1/2}$ for $X\in\mathbb{R}^{d\times r}$ with $X\,\mathtt{invroot}(X^\top X)$. The latter is slightly faster for small $r$ (since $X^\top X\in\mathbb{R}^{r\times r}$), but we found the benefits to be small in practice.

\begin{algorithm}[H]
\caption{Stable sMuon}
\label{alg:matmul_only_stable}
\KwIn{Same as Alg.~\ref{alg:conceptual}; jitter $\epsilon > 0$ (default to $10^{-4}$)}
\DontPrintSemicolon
\tcp{Inverse roots}
$D_B \leftarrow B^\top B$\;
$\epsilon_B \leftarrow \epsilon\max(\mathrm{tr}(D_B)/r, 1)$\;
$S_B \leftarrow \mathtt{invroot}(D_B+\epsilon_B I)$\;
$D_A \leftarrow AA^\top$\;
$\epsilon_A \leftarrow \epsilon\max(\mathrm{tr}(D_A)/r, 1)$\;
$S_A \leftarrow \mathtt{invroot}(D_A+\epsilon_A I)$\;
\tcp{Bases}
$X \leftarrow S_B^2 M_A$\;
$V_1 \leftarrow \msign(A^\top)$\;

$\tilde Y \leftarrow M_B - B\big(S_B^2(B^\top M_B)\big)$\;
$\tilde Y \leftarrow \tilde Y - B\big(S_B^2(B^\top \tilde Y)\big)\quad$\tcp{two orthog. steps}
$U_2 \leftarrow \msign(\tilde Y)$, \quad $Y \leftarrow \tilde Y\,S_A^2$\;
$\tilde Z \leftarrow M_A^\top - V_1(V_1^\top M_A^\top)$\;
$V_2 \leftarrow \msign(\tilde Z - V_1(V_1^\top \tilde Z))\quad$\tcp{two orthog. steps}
$C \leftarrow \begin{bmatrix} S_B(M_A V_1) & S_B(M_A V_2)\\ (U_2^\top Y)(A V_1) & 0 \end{bmatrix}$\;
$\Omega \leftarrow \msign(C)$\;

\tcp{Step direction}

$P \leftarrow \tfrac12\big(AV_1 + (AV_1)^\top\big)$\tcp{symmetric matrix}
$T \leftarrow \tfrac12\big(P S_A^2 + (P S_A^2)^\top\big)\quad$\tcp{symmetric matrix}

$\delta A \leftarrow S_B(\Omega_{11}V_1^\top + \Omega_{12}V_2^\top)$\;
$\delta B \leftarrow U_2\,\Omega_{21}\,T$\;

\tcp{Update parameters}
$s \leftarrow \sqrt{1-\lambda\eta}$\;
$B' \leftarrow sB - \tfrac{\eta}{s}\delta B$\;
$A' \leftarrow sA - \tfrac{\eta}{s}\delta A$\;

\tcp{Project momentum}
$M_A \leftarrow ({B'}^\top B)X + ({B'}^\top Y)A$\;
$M_B \leftarrow B(X{A'}^\top) + Y(A{A'}^\top)$\;
$B \leftarrow B'$\;
$A \leftarrow A'$\;
\end{algorithm}
\section{Comparison of Muon-based LoRA optimizers}\label{app:comparison}
Every method that combines Muon with LoRA has the same obstruction: Muon requires a step along
$\msign(G_{\Delta W})$, but a gradient step for a LoRA adapter is constrained inside the set $\mathcal T=\{Z:(I-P_B)Z(I-P_A)=0\}$. The ideal step $\msign(G_{\Delta W})$ is not in general within this set, thus different methods make different approximations. We compare the different methods in Table~\ref{tab:muon-lora-geometry}.
\begin{table}[htbp]
\centering
\footnotesize
\setlength{\tabcolsep}{3.5pt}
\caption{Comparison of the Muon-based LoRA optimizers, all written in the convention $\Delta W=BA$ of
this paper. $\Pi_{\mathcal T}(\cdot)$ is the orthogonal projector onto the realizable step
directions.}
\label{tab:muon-lora-geometry}
\makebox[\linewidth][c]{%
\begin{tabular}{@{}>{\raggedright\arraybackslash}p{1.6cm}
                  >{\raggedright\arraybackslash}p{3.4cm}
                  >{\raggedright\arraybackslash}p{3.8cm}
                  >{\raggedright\arraybackslash}p{2.0cm}
                  >{\raggedright\arraybackslash}p{3.3cm}@{}}
\toprule
Method & $\delta\Delta W$ & Momentum & Weight decay on $\Delta W$ & Expensive operations \\
\midrule
Per-factor Muon\newline\citep{qu2026can,kang2026uniform}
  & $-\eta\big(\msign(M_B)A+B\,\msign(M_A)\big)+O(\eta^2)$
  & Standard per-factor EMAs, i.e. accumulate $G_B$ in $M_B$ ($M_B\!\leftarrow\!\beta M_B+(1{-}\beta)G_B$), likewise for $M_A$
  & Decoupled, per factor: $\Delta W\mapsto(1-\lambda\eta)^2\Delta W$
  & 2 orthogonalizations\\
\addlinespace
Spectron\newline\citep{janson2026stabilizing}
  & As above with $\eta$ replaced by $\rho=\eta/(\|B\|_2{+}\|A\|_2{+}1)$; only $\|\delta\Delta W\|_2\le\eta$ is controlled, no direction matched
  & Standard per-factor EMAs
  & Decoupled, per factor
  & 2 orthogonalizations, 2 power iterations to estimate $\|B\|_2,\|A\|_2$ \\
\addlinespace
LoRA-Muon\newline\citep{cesista2026lora}
  & $-\tfrac{\eta}{2}\big(\msign(P_BG_{\Delta W})+\msign(G_{\Delta W}P_A)\big)+O(\eta^2)$
  & Standard per-factor EMAs
  & Split: let $s=\sqrt{1-\lambda\eta}$, and $A \leftarrow s A - \eta / s\, \delta A$ (same for $B$). Overall gives $-\lambda\eta\,\Delta W + O(\eta^2)$ (Sec.~\ref{sec:wd})
  & The same 2 orthogonalizations, 2 $r\times r$ inverse square roots of $(B^\top B)^{-1/2}$ and $(AA^\top)^{-1/2}$ \\
\addlinespace
Riemannion\newline\citep{bogachev2025lora}
  & $\mathrm{SVD}_r\big(\Delta W-\eta\,\Pi_{\mathcal T}(\msign( H))\big)-\Delta W$
  & $A$ and $B$ semi-orthonormal. Accumulate $(I-BB^\top)G_B$ in $M_B$, and accumulate $G_A$ in $M_A$. Let $ H=BM_A+(I-BB^\top)M_BA$. After the retraction, set $M_B=(I-BB^\top) HA^\top$, $M_A=B^\top H$
  & Prior to retraction, $A$ is decayed via $-\eta\gamma\,A$; $B$ is not decayed.
  & 4 thin QRs, 2 SVDs ($2r{\times}2r$) \\
\addlinespace
Ours
  & $-\eta (I - P_B) \msign(H) P_A - \eta P_B \msign(H)   +O(\eta^2)$
  & Standard per-factor EMAs, but let $H=(B^\dagger)^\top M_A + (I - P_B) M_B (A^\dagger)^\top$. After the gradient step, set $M_B =  H  A^\top $, $M_A = B^\top  H$
  & Split
  & 2 $r{\times}r$ inverse square roots, 4 orthogonalizations ($d\times r$, worst case) \\
\bottomrule
\end{tabular}}
\end{table}
\FloatBarrier
\section{Optimality of the sMuon step}\label{app:targeting_H}
We show in Lemma~\ref{lem:step} that our particular choice of $\delta A$ and $\delta B$ (a) is, to first order, as close as possible to the target step $-\eta\msign(H)$, and (b) while the minimum is not unique in terms of steps $(\delta B,\, \delta A)$, it is unique in terms of the first-order change to the whole linear layer.
\begin{lemma}\label{lem:step}
Let $B\in\mathbb R^{d_1\times r}$ have full column rank and $A\in\mathbb R^{r\times d_2}$ have full row rank.
Let $P_B = BB^\dagger$ and $P_A = A^\dagger A$, let $G_{\Delta W}$ be the adapter gradient, and let $H = G_{\Delta W} - (I-P_B)\,G_{\Delta W}\,(I-P_A)$.
Write $\delta\Delta W = (B+\delta B)(A+\delta A) - BA$ and $\delta\Delta W^\star = -\eta\,\msign(H)$.
Consider updates of the form $\delta B = O(\eta)$, $\delta A = O(\eta)$, then
\begin{equation}\label{eq:step_expand}
\|\delta\Delta W^\star - \delta\Delta W\|_F \;=\; \,\|\eta\msign(H) + \delta B\,A + B\delta A\|_F \;+\; O(\eta^2),
\end{equation}
\begin{enumerate}
\item\label{lem:step_ls} The choice $\delta A = -\eta\,B^\dagger\,\msign(H)$, $\delta B = -\eta\,(I-P_B)\,\msign(H)\,A^\dagger$
minimizes $\|\delta\Delta W^\star - \delta\Delta W\|_F$ to first order.
\item\label{lem:step_gauge} The first-order minimizer is not unique. If $(\delta B, \,\delta A)$ is a minimizing pair, then so is $(\delta B - BS,\ \delta A + SA)$ with $S\in\mathbb R^{r\times r}$, provided that $S=O(\eta)$.
Every first-order minimizer realizes the same overall update,
\begin{equation}\label{eq:step_update}
\delta\Delta W \;=\; -\eta\big[P_B\msign(H) + (I-P_B)\msign(H)P_A\big] + O(\eta^2).
\end{equation}
\end{enumerate}
\end{lemma}
\begin{proof}
Since $B^\dagger B = I = AA^\dagger\in\mathbb{R}^{r\times r}$, we have $(I-P_B)B = 0$ and $A(I-P_A) = 0$.
Write $M = \msign(H)$, $\delta B = \eta b$, $\delta A = \eta a$ with $a,b$ bounded independently of $\eta$. Then
$$\delta\Delta W^\star - \delta\Delta W \;=\; -\eta\,(M + bA + Ba) +O(\eta^2)$$
and the reverse triangle inequality gives Eq.~\eqref{eq:step_expand}.
It therefore suffices to minimize the leading coefficient $\|M + bA + Ba\|_F$.

Any $X \in \mathbb{R}^{d_1 \times d_2}$ admits the decomposition
$$X = P_B X P_A + P_B X (I-P_A) + (I-P_B) X P_A + (I-P_B) X (I-P_A),$$
a sum of four mutually Frobenius-orthogonal terms. We have $(I - P_B)(bA + Ba)(I - P_A) = 0$,  so
$$\|M + bA + Ba\|_F^2 \;\geq\; \|(I-P_B)\,M\,(I-P_A)\|_F^2.$$
Crucially, the RHS is independent of the choice of $b,\,a$.
We get equality when the other three blocks of the decomposition are zero, which they
are for $M + bA + Ba$,
\begin{align*}
P_B &(M + bA + Ba) P_A + P_B (M + bA + Ba)(I - P_A) + (I - P_B) (M + bA + Ba)P_A\\
&= P_B (M + bA + Ba) \;+\; (I - P_B) (M + bA + Ba) P_A\\
&= \big[P_B M - P_B M\big] \;+\; \big[(I-P_B) M P_A - (I-P_B) M P_A\big] \;=\; 0,
\end{align*}
where we substituted the choice of Part~\ref{lem:step_ls},
$$bA = -(I-P_B)M\,\big[A^\dagger A\big] = -(I-P_B)MP_A, \qquad Ba = -\big[BB^\dagger\big]\,M = -P_BM,$$
and used $P_B(I-P_B) = 0$, $(I-P_B)^2 = I-P_B$, $P_A^2 = P_A$, etc..
Since $\|\delta B\|_F, \|\delta A\|_F = O(\eta)$, this pair lies in the domain and attains equality, so it minimizes the leading coefficient, and hence $\|\delta\Delta W^\star - \delta\Delta W\|_F$ to first order.

For Part~\ref{lem:step_gauge}, equality holds only when the three blocks vanish. Since the fourth block of $bA + Ba$ is identically zero by $(I-P_B)B = 0$ and $A(I-P_A) = 0$, the three conditions force
$$bA + Ba \;=\; -\,P_B M - (I-P_B) M P_A \qquad \text{for every minimizing pair.}$$
Consequently, every first-order minimizer satisfies
$$\delta\Delta W \;=\; \eta\,(bA + Ba) + \delta B\,\delta A \;=\; -\eta\big[P_B M + (I-P_B) M P_A\big] + O(\eta^2),$$
which is Eq.~\eqref{eq:step_update}. Finally, if $S = O(\eta)$, the pair $(\delta B - BS,\ \delta A + SA)$ remains $O(\eta)$ and satisfies
$$(\delta B - BS)\,A + B\,(\delta A + SA) = \delta B\,A + B\,\delta A,$$
so it has the same leading coefficient and is again a first-order minimizer.
\end{proof}

\FloatBarrier
\section{Learning rate curves}\label{app:lr_sweep}
We show learning rate curves for the commonsense runs in Figures~\ref{fig:cs-accuracy-grid} (overall commonsense accuracy) and~\ref{fig:cs-val-loss-grid} (val. loss). We swept different learning rates for different optimizers, since we found that the optimal learning rate varies between optimizers. The accuracy figures shown in the main text are chosen by taking the optimal learning rate according to loss on a held-out validation set.
\begin{figure}[t]
  \centering
  \includegraphics[width=\linewidth]{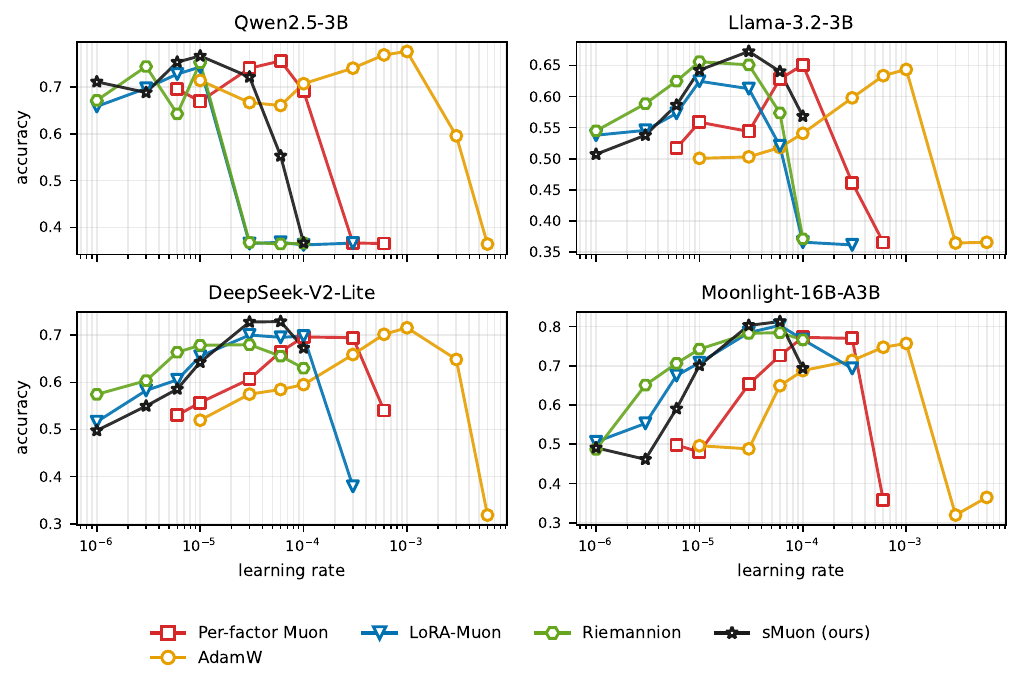}
  \caption{Common-sense accuracy against learning rate for each base model.
  Each curve is one optimizer; accuracy is the overall common-sense benchmark
  accuracy at the final training step.}
  \label{fig:cs-accuracy-grid}
\end{figure}

\begin{figure}[t]
  \centering
  \includegraphics[width=\linewidth]{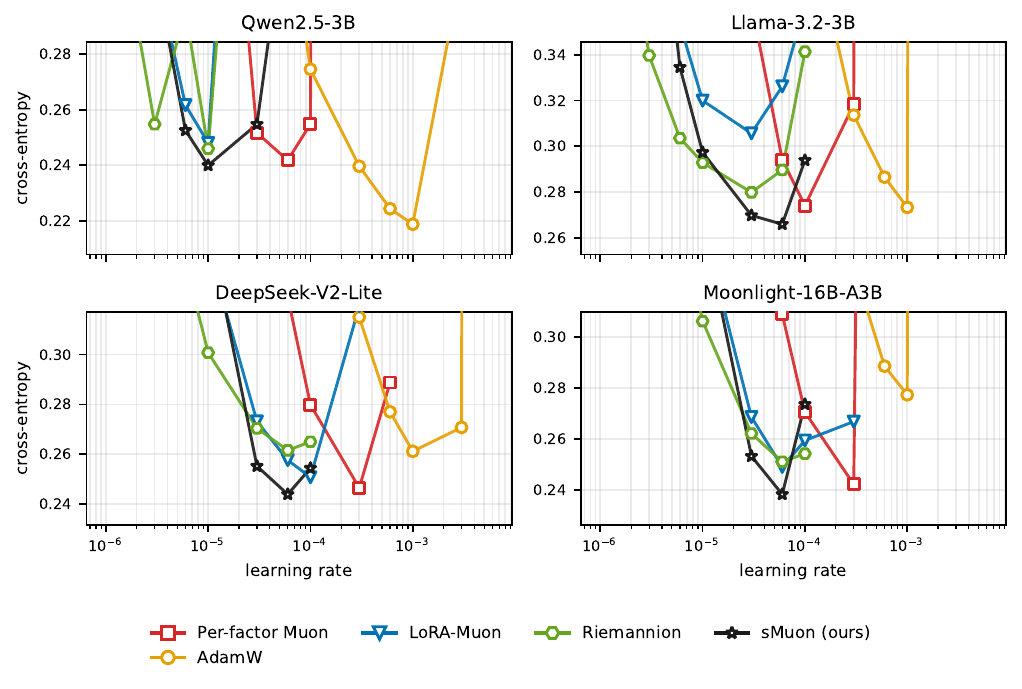}
  \caption{Validation loss against learning rate for each base model, over the
  same runs as Figure~\ref{fig:cs-accuracy-grid}.}
  \label{fig:cs-val-loss-grid}
\end{figure}
\end{document}